%% file: main.tex
\documentclass[conference]{IEEEtran}
\IEEEoverridecommandlockouts
\usepackage{cite}
\usepackage{amsmath,amssymb,amsfonts}
\usepackage{algorithmic}
\usepackage{algorithm}
\usepackage{graphicx}
\usepackage{textcomp}
\usepackage{xcolor}
\usepackage{microtype}
\usepackage{graphicx}
\usepackage{subfigure}
\usepackage{booktabs}
\usepackage{hyperref}
\usepackage{subfigure,graphicx,dsfont,amssymb,amsmath,cite,calc,tabulary}
\usepackage{xcolor}
\usepackage[english]{babel}
\usepackage[utf8]{inputenc}

\usepackage{amsmath}
\usepackage{amssymb}
\usepackage{amsthm}
\usepackage{fixmath}
\usepackage{xcolor}
\newtheorem{theorem}{Theorem}

\newtheorem{lemma}[theorem]{Lemma}
\input{math_commands.tex}

\def\BibTeX{{\rm B\kern-.05em{\sc i\kern-.025em b}\kern-.08em
    T\kern-.1667em\lower.7ex\hbox{E}\kern-.125emX}}

\begin{document}

\title{Graph-Assisted Communication-Efficient \\Ensemble Federated Learning}
	\author{
        \IEEEauthorblockN{Pouya M.~Ghari and Yanning Shen}
        \IEEEauthorblockA{\textit{Department of Electrical Engineering and Computer Science} \\
        \textit{University of California, Irvine}\\
        Irvine, CA, USA \\
 E-mails: pmollaeb@uci.edu, yannings@uci.edu
 }
	}

\maketitle

\begin{abstract}
Communication efficiency arises as a necessity in federated learning due to limited communication bandwidth. To this end, the present paper develops an algorithmic framework where an ensemble of pre-trained models is learned. At each learning round, the server selects a subset of pre-trained models to construct the ensemble model based on the structure of a graph, which characterizes the server's confidence in the models. Then only the selected models are transmitted to the clients, such that certain budget constraints are not violated. Upon receiving updates from the clients, the server refines the structure of the graph accordingly. The proposed algorithm is proved to enjoy sub-linear regret bound. Experiments on real datasets demonstrate the effectiveness of our novel approach.
\end{abstract}


\section{Introduction}
Prevalence of distributed networks consisting of devices such as mobile phones and sensors with growing computational and storage capability enables pushing more network computations to the edge. \emph{Federated learning} has emerged as a promising framework to train machine learning models under orchestration of a central server while training data remains distributed among the edge devices which are called \emph{clients} \cite{McMahan2017}. 
In federated learning, a central server sends the current model to a set of clients at each learning round. Participating clients then compute updates of the current model based on their local data and send these updates to the server instead of their local data. The server then update the model. This procedure continues until convergence. In this context, communication efficiency is of utmost importance. To this end,  clients-to-server
communication efficiency in federated learning  has been  studied extensively in the context of model updates compression, see e.g., \cite{konency2017,Rothchild2020}.
On the other hand, server-to-clients communication bottleneck arises if the learning task involves large model, such that the required bandwidth to transmit the model exceeds the available bandwidth for server-to-clients communication. For example, base stations can be employed as an aggregating server in certain applications \cite{Abad2019}. Often times the base station may only assign a limited portion of available bandwidth for server-to-client communication, while reserving most of the available bandwidth for other service required by users. Meanwhile, large models trained at the central server may exceed the clients' memory capacity. These challenges all motivate the study of server-to-clients communication-efficient federated learning, which is also the focus of the present work.

Ensemble learning methods are known to be effective for learning large-scale models, which combine several base predictors or experts to generate more accurate ensemble model. However, conventional ensemble learning methods (see e.g. \cite{Buhlmann2012}) are not directly applicable for communication efficient federated learning. To adapt ensemble learning to federated setting, FedBoost has been proposed by \cite{Hamer2020}, where the server constructs a model by combining a subset of pre-trained models. In this way, the server only needs to transmit a subset of pre-trained models to the clients at each learning round. Pre-trained models can be trained on publicly available data without observing clients' data. FedBoost imposes a budget constraint on the number of pre-trained models that can be transmitted to the clients; however it cannot guarantees that  the budget constraint is not violated at each learning round. Instead, it only guarantees the expected cost of model transmission satisfies the budget constraint. Moreover, ensemble learning techniques has been employed in vertical federated learning \cite{Chen2021}. 

The present paper studies server-to-clients communication efficiency in federated learning. Specifically, we aim at selecting a subset of pre-trained models to construct an ensemble model. To this end, each pre-trained model is viewed as an expert lying on a graph. At each learning round, the server chooses a subset of pre-trained models based on the structure of the graph. Upon receiving the updates from the clients, the server refines the structure of the graph. In this context, the prediction provided by each pre-trained model can be viewed as feedback given by the expert. Hence, the constructed graph is named \emph{feedback graph}. The cost of each pre-trained model is  proportional to its parameter size and a budget constraint is set to  transmit the models in order to construct the ensemble model. We develop an algorithm called \textbf{e}nsemble \textbf{f}ederated \textbf{l}earning with \textbf{f}eedback \textbf{g}raph (EFL-FG) which selects a subset of pre-trained models to be transmitted to clients, and guarantees the resulting communication cost does not exceed budget constraint at each learning round.  EFL-FG is proved to enjoy sub-linear regret. Experiments on  real datasets showcase the effectiveness of our proposed algorithms compared with state-of-art ensemble federated learning alternatives.


\section{Problem Statement and Preliminaries}
This section introduces the problem of federated learning with ensemble method. Let there are a set of $N$ clients that the server interacts with them to perform a learning task. Moreover, $\gX$ and $\gY$ denote the input space and output space, respectively such that a data sample $(\vx,y) \in \gX \times \gY$. Furthermore, there are a set of $K$ pre-trained models $f_1(\cdot),\ldots,f_K(\cdot)$ at the server. Each model $f_k(\cdot): \gX \rightarrow \gY$, $\forall k \in [K]$ is a mapping from the input space to the output space where $[K]$ denotes the set $\{1,\ldots,K\}$. At each round of learning, the server uniformly chooses a random subset of clients to send them the current model. The number of clients chosen by the server depends on the available bandwidth for clients-to-server communication. Specifically, at each learning round, the server constructs a model using pre-trained models and send it to the chosen subset of clients. Then, upon receiving new data samples, clients perform the learning task using the current model received from  the server. 

Let $\sC_t$ denote the set of clients selected by the server at learning round $t$, and $\gS_t := \{(\vx_{i,t},y_{i,t}), \forall i \in \sC_t \}$ represents a set of data samples observed by the chosen subset of clients at learning round $t$. In this case, the problem of federated learning can be viewed as a function approximation problem. Specifically, given data samples $\{\gS_t\}_{t=1}^T$, the goal is to find the function $\hat{f}(\cdot)$, such that the difference between $\hat{f}(\vx_{i,t})$ and $y_{i,t}$ is minimized. When the $i$-th client performs the learning task on the received datum $\vx_{i,t}$, it incurs the loss $\gL(\hat{f}(\vx_{i,t}),y_{i,t})$ where $\gL(\cdot,\cdot)$ denotes the loss function.  And $\gL(f_k(\vx_{i,t}),y_{i,t})$ denote the loss of each model $f_k(\cdot)$, $\forall k \in [K]$ for the datum $\vx_{i,t}$ associated with the $i$-th client. In this context, the goal of the server is to construct the function $\hat{f}(\cdot)$ using pre-trained models in a way that the cumulative loss is minimized. In order to build the function $\hat{f}(\cdot)$ using the pre-trained models, the server employs the ensemble method. Using the conventional ensemble method, the model at the server is constructed as
\begin{align}
    \hat{f}(\vx) = \sum_{k=1}^{K}{w_k f_k(\vx)}, \sum_{k=1}^{K}{w_{i}} = 1. \label{eq:1}
\end{align}
The ensemble method in \eqref{eq:1} requires that the server sends all models to the clients. However, this may not be feasible due to e.g., insufficient bandwidth for server-to-clients communication, and limited computational capability and memory of clients to store all models. The present paper proposes a novel algorithmic framework to choose a data-driven subset of models in a way that these limitations are taken into account.

\section{Ensemble Federated Learning with Graphs} \label{sec:EFL-FG}
The present section first introduces a disciplined way to construct a graph based on the performance of pre-trained models. Then, a novel algorithm is proposed to construct an ensemble model employing a subset of pre-trained models chosen by the server based on the graph.

Let $c_k$ be the cost incurred when the $k$-th pre-trained model is transmitted. Specifically, $c_k$ can be the bandwidth required for transmission of the $k$-th model to clients. Let $B_t$ denote the budget of the server which denotes the cumulative cost the server can afford for transmission at learning round $t$, e.g., available bandwidth for server-to-clients communication. In what follows, a principled algorithm to construct a graph is proposed which assists the learner to obtain an ensemble model.
\subsection{Feedback Graph Generation}
Let $\gG_t=(\gV,\gE_t)$ be a directed graph at learning round $t$ with a set of vertices $\gV$ and a set of edges $\gE_t$. Each vertex $v_k \in \gV$, $\forall k \in [K]$ represents the pre-trained model $f_k(\cdot)$. 
Let $\vw_t$ be a weight vector, where the $k$-th element $w_{k,t}$ is the weight associated with the $k$-th model $f_k(\cdot)$, indicating the server's confidence about the performance of model $f_k(\cdot)$. At each learning round, the server updates $\vw_t$ based on the observed loss of $f_k(\cdot)$ which will be specified later. Let $\sN_{k,t}^\text{out}$, $\forall k \in [K]$ be the out-neighbor set of $v_k$. In order to construct the set $\sN_{k,t}^\text{out}$, $\forall k \in [K]$, the server appends nodes $v_j$ to $\sN_{k,t}^\text{out}$ based on both weights and costs of models such that the cumulative cost of nodes in $\sN_{k,t}^\text{out}$ does not exceed the budget $B_t$. At first, the server append $v_k$ to $\sN_{k,t}^\text{out}$ which means that there is a self loop for each $v_k \in \gV$. Let 
\begin{align}
    \sM_{k,t} := \{ v_i &|\forall i: \sum_{j \in \sN_{k,t}^\text{out}}{c_j} + c_i \le B_t, \nonumber \\ & \sum_{j \in \sN_{k,t}^\text{out}}{w_j} + w_i \le \!\!\sum_{j \in \sN_{k,t-1}^\text{out}}{\!\!w_j}, v_i \notin \sN_{k,t}^\text{out} \} \label{eq:6}
\end{align}
denote a set of vertices associated with $v_k$ at learning round $t$. At learning round $t$, find
\begin{align}
    v_d = \arg\max_{v_i \in \sM_{k,t}}{\frac{w_{i,t}}{\sum_{v_j \in \sN_{k,t}^\text{out}}{c_j} + c_i}}. \label{eq:2}
\end{align}
the set $\sM_{k,t}$ then is updated by appending $v_d$ to $\sN_{k,t}^\text{out}$. This procedure continues until $\sM_{k,t}$ becomes an empty set, i.e., $|\sM_{k,t}| = 0$, where $|\cdot|$ represents the cardinality of a set. This means there is no more node that can be appended to $\sN_{k,t}^\text{out}$ such that the constraints in \eqref{eq:6} are satisfied. Moreover, according to \eqref{eq:2}, the server appends $v_d$ to $\sN_{k,t}^\text{out}$ by considering the trade-off between the performance of nodes in prior rounds and the amount of cost that they might add to current cumulative cost of out-neighbors of $v_k$. When the server constructs $\sN_{k,t}^\text{out}$, $\forall k \in [K]$, the set of edges $\mathcal{E}_t$ can be constructed. Specifically, $(k,j) \in \mathcal{E}_t$ if $v_j \in \sN_{k,t}^\text{out}$.  The procedure to construct the graph $\gG_t$ is summarized in \Algref{alg:1}. At each learning round, the server draws one node in $\gG_t$ and transmits models which are out-neighbors of the chosen one. Then the learning task is carried out with a subset of models which are considered as nodes in $\gG_t$. Thus, output of the selected  models can be viewed as feedback collected from $\gG_t$, which is henceforth named as \emph{feedback graph}.
\begin{algorithm}[tb]
	\caption{Feedback Graph Generation}
	\label{alg:1}
	\begin{algorithmic}
		\STATE {\bfseries Input:}{Models $f_k(.)$, weights $w_{k,t}$, costs $c_k$, $\forall k \in [K]$ and the budget $B_t$. }
		\FOR{$k=1,...,K$}
		\STATE Append $v_k$ to $\sN_{k,t}^\text{out}$.
		\WHILE{$|\sM_{k,t}|>0$}
		\STATE Node $v_k$ appends $v_d$ as in \eqref{eq:2} to $\sN_{k,t}^\text{out}$.
		\STATE Update $\sM_{k,t}$ with respect to updated $\sN_{k,t}^\text{out}$.
		\ENDWHILE
		\ENDFOR
		\STATE {\bfseries Output:}{Feedback Graph $\gG_t=(\gV,\gE_t)$.}
	\end{algorithmic}
\end{algorithm}

\subsection{Ensemble Federated Learning}

At each learning round $t$, the server selects one node in $\gG_t$ and constructs the ensemble model with out-neighbors of the chosen node. To this end, the server assigns weight $u_{k,t}$ to node $v_k$ which indicates the confidence in the accuracy of the obtained ensemble model when node $v_k$ is chosen. Then, the server draws one of the nodes $v_k \in \gV$ according to the probability mass function (PMF) $\vp_t$ as follows
\begin{align}
    p_{k,t} = (1-\xi)\frac{u_{k,t}}{U_t} + \frac{\xi}{|\sD_t|}\gI(v_k \in \sD_t) \label{eq:3}
\end{align}
where $\gI(\cdot)$ denotes the indicator function, $\xi$ is the exploration rate and $U_t := \sum_{k=1}^{K}{u_{k,t}}$. The set $\sD_t$ denotes a dominating set for the feedback graph $\gG_t$. The PMF in \eqref{eq:3} constitutes a trade-off between exploitation and exploration. Let $I_t$ denote the index of the drawn node at learning round $t$. Let $\sS_t$ be a set of indices of nodes which are out-neighbors of the chosen node $v_{I_t}$. In this case, the server utilizes the weighting vector $\vw_t$ to construct the ensemble model using models whose indices are in $\sS_t$ as follows
\begin{align}
    \hat{f}_t(\vx) = \sum_{k \in \sS_t}{\frac{w_{k,t}}{W_t}f_k(\vx)} \label{eq:4}
\end{align}
where $W_t := \sum_{k \in \sS_t}{w_{k,t}}$. Then, the server transmits the ensemble model along with the chosen subset of pre-trained models to a subset of $N_t$ clients, which is chosen uniformly at random. Upon receiving new datum the chosen subset of clients perform the learning task using the ensemble model sent by the server. Then, the $i$-th client where $i \in \sC_t$, incurs loss $\gL(\hat{f}(\vx_{i,t}),y_{i,t})$ associated with the received data sample $(\vx_{i,t},y_{i,t})$. Furthermore, the $i$-th client ($\forall i \in \sC_t$) computes the loss $\gL(f_k(\vx_{i,t}),y_{i,t})$, $\forall k \in \sS_t$. Then each client in $\sC_t$ transmits the losses associated with the ensemble model and the chosen subset of models to the server. Upon receiving the losses, the server updates  $w_{k,t}$ and $u_{k,t}$, $\forall k \in [K]$. To this end, the server employs the importance sampling loss estimate which results in an unbiased estimation of the incurred loss. The importance sampling loss estimate for the $k$-th model can be expressed as
\begin{align}
    \ell_{k,t} = \frac{\sum_{i \in \sC_t}{\gL(f_k(\vx_{i,t}),y_{i,t})}}{q_{k,t}}\gI(k \in \sS_t) \label{eq:5}
\end{align}
where 
\begin{align}
    q_{k,t} := \sum_{v_j \in \sN_{k,t}^\text{in}}{p_{j,t}} \label{eq:7}
\end{align}
represents the probability of $k \in \sS_t$, where $\sN_{k,t}^\text{in}$ denotes the in-neighbor set of $v_k$ in $\gG_t$.
In addition, define the importance sampling loss estimate associated with incurred loss of ensemble model when $v_k$ is drawn by the server as
\begin{align}
    \hat{\ell}_{k,t} = \frac{\sum_{i \in \sC_t}{\gL(\hat{f}(\vx_{i,t}),y_{i,t})}}{p_{k,t}}\gI(k=I_t). \label{eq:8}
\end{align}
Using the importance sampling loss estimates in \eqref{eq:5} and \eqref{eq:8}, the weights $w_{k,t}$ and $u_{k,t}$ can the be updated as follows
\begin{subequations} \label{eq:9}
    \begin{align}
        w_{k,t+1} = w_{k,t}\exp(-\eta \ell_{k,t}) \label{eq:9a} \\
        u_{k,t+1} = u_{k,t}\exp(-\eta \hat{\ell}_{k,t}) \label{eq:9b}
    \end{align}
\end{subequations}
where $\eta$ is the learning rate. The procedure that the server sends subset of models to clients is summarized in \Algref{alg:2}. The algorithm is called EFL-FG which stands for Ensemble Federated Learning with Feedback Graph. In each learning round, clients only need to send their computed losses to the server while they do not have to reveal the loss function $\gL(\cdot,\cdot)$ and observed data samples to the server. Specifically, in some applications data samples $(\vx_{i,t},y_{i,t})$ may include some information about clients that they do not wish to share with the server.
\begin{algorithm}[tb]
	\caption{EFL-FG: Ensemble Federated Learning with Feedback Graph}
	\label{alg:2}
	\begin{algorithmic}
		\STATE {\bfseries Input:}{Models $f_k(\cdot)$, weights $w_{k,t}$, costs $c_k$, $\forall k \in [K]$. }
		\STATE \textbf{Initialize:} $w_{k,1} = 1$, $u_{k,1} = 1$, $\forall k \in [K]$. \noindent
		\FOR{$t=1,...,T$}
		\STATE The server generates  $\gG_t$ using \Algref{alg:1}.
		\STATE The server draws one node in $\gG_t$ according to the $\vp_t$ in \eqref{eq:3}, with out-neighbors indexed by $\sS_t$.
		\STATE The server randomly selects a subset of clients $\sC_t$.
		\STATE The server sends models in $\sS_t$ and the ensemble model to the clients in $\sC_t$.
		\STATE Clients in $\sC_t$ compute and send back to the sever the losses  $\gL(\hat{f}(\vx_{i,t}),y_{i,t})$ and  $\gL(f_k(\vx_{i,t}),y_{i,t})$, $\forall k \in \sS_t$.
		\STATE The server computes the importance sampling loss estimates $\ell_{k,t}$ and $\hat{\ell}_{k,t}$, $\forall k \in [K]$.
		\STATE The server updates $w_{k,t+1}$ and $u_{k,t+1}$, $\forall k$ as in \eqref{eq:9}.
		\ENDFOR
	\end{algorithmic}
\end{algorithm}

The present paper considers the case where the number of selected clients may vary with learning rounds. According to \Algref{alg:2}, clients in $\sC_t$ need to send computed losses of sent models along with the losses of ensemble model to the server. Therefore, according to available bandwidth for clients-to-server communication and required bandwidth for sending the updates, the server determines the number of selected clients $N_t$ such that clients can send computed losses without interfering each other. Specifically, the available bandwidth should be divided between clients in $\sC_t$ without overlap to prevent interference. Let $b_t$ be the available bandwidth for clients-to-server communication at learning round $t$. In this case, the server can obtain the number of clients as $N_t \le \left \lfloor{\frac{b_t}{b_\ell ({|\sN_{I_t,t}^\text{out}|}+1)}} \right \rfloor$, where $b_\ell$ is the required bandwidth for transmission of each loss.

\textbf{Comparison with online learning.} Online learning studies problems where a learner interacts with a set of experts such that at each learning round the learner makes decision based on advice received from the experts \cite{Cesa-Bianchi2006,Auer2003}. The learner may observe the loss associated with a subset of experts after decision making, which can be modeled using a feedback graph \cite{Mannor2011,Cortes2020}. In EFL-FG, each pre-trained model is also viewed as an expert. However, there is a major innovative difference compared with online learning with feedback graph: the proposed EFL-FG constructs and refines the feedback graph to improve the performance while in online learning, the feedback graph is generated in an adversarial manner.


\subsection{Regret Analysis}
The present subsection studies the performance of EFL-FG in terms of cumulative regret. The difference between the loss incurred by the ensemble model and the loss of the best pre-trained model in the hindsight is defined as the regret of the ensemble model. In this context, the best pre-trained model in the hindsight is the one with minimum cumulative loss among all pre-trained models. Therefore, the cumulative regret of EFL-FG can be written as
\begin{align}
    \gR_T = & \sum_{t=1}^{T}{\sum_{i \in \sC_t}{\E_t[\gL(\hat{f}(\vx_{i,t}),y_{i,t})]}} \nonumber \\ & - \min_{k \in [K]}{\sum_{t=1}^{T}{\sum_{i \in \sC_t}{\gL(f_k(\vx_{i,t}),y_{i,t})}}} \label{eq:10}
\end{align}
where $\E_t[.]$ denotes the conditional expectation given observed losses in prior learning rounds. In order to analyze the performance of EFL-FG, we assume that the following conditions hold:\\
\textbf{(a1)} The loss function $\gL(f_k(\vx_{i,t}),y_{i,t})$ is convex with respect to $f_k(\vx_{i,t})$.\\ 
\textbf{(a2)} For each $(\vx_{i,t},y_{i,t})$, the loss is bounded $0 \le \gL(f_k(\vx_{i,t}),y_{i,t}) \le 1$.\\ 
\textbf{(a3)} The budget satisfies $B_t \ge c_k$, $\forall k \in [K]$, $\forall t$.\\
The following Theorem presents the regret bound for EFL-FG with respect to the best pre-trained model in hindsight.
\begin{theorem} \label{th:1}
 Under (a1)--(a3), the expected cumulative regret of EFL-FG is bounded by
\begin{align}
    & \sum_{t=1}^{T}{\sum_{i \in \sC_t}{\E_t[\gL(\hat{f}(\vx_{i,t}),y_{i,t})]}} - {\sum_{t=1}^{T}{\sum_{i \in \sC_t}{\gL(f_{k^*}(\vx_{i,t}),y_{i,t})}}} \nonumber \\ \le & \frac{\ln (K|\sN_{k^*,1}^\text{out}|)}{\eta} \nonumber \\ & + \sum_{t=1}^{T}{\left(\xi(1-\frac{\eta}{2}|\sC_t|^2) + \frac{\eta}{2}(K+\frac{1}{\bar{q}_{k^*,t}})|\sC_t|^2\right)} \label{eq:11}
\end{align}
where $k^* = \arg\min_{k \in [K]}{\sum_{t=1}^{T}{\sum_{i \in \sC_t}{\gL(f_k(\vx_{i,t}),y_{i,t})}}}$ is the index of the best pre-trained model in the hindsight, and $\frac{1}{\bar{q}_{k^*,t}} := \sum_{j \in \sN_{k^*,t}^\text{out}}\frac{w_{j,t}}{q_{j,t}W_{k^*,t}}$.
\end{theorem}
\begin{proof}
see Appendix.
\end{proof}
According to \eqref{eq:3}, $p_{k,t} > \frac{\xi}{|\sD_t|}$, $\forall k \in \sD_t$. In addition, each node $v_k \in \gV$ is in-neighbor to at least one vertex in the dominating set $\sD_t$. Therefore, based on \eqref{eq:7}, it can be concluded that $q_{k,t} > \frac{\xi}{|\sD_t|}$, $\forall k \in [K]$ and as a result, we have $\frac{1}{\bar{q}_{k^*,t}} < \frac{|\sD_t|}{\xi}$. If the greedy set cover algorithm (see e.g. \cite{Chvatal1979}) is employed to find a dominating set for the feedback graph $\gG_t$, EFL-FG obtains a dominating set with $|\sD_t| = \gO(\alpha(\gG_t)\ln K)$ where $\alpha(\gG_t)$ denotes the independence number of the feedback graph $\gG_t$ \cite{Alon2017}. In this case, if the server sets $\eta = \gO(\sqrt{\frac{\ln K}{T}})$, $\xi = \gO(\frac{(\ln K)^{\frac{3}{4}}}{T^{\frac{1}{4}}})$ and $|\sC_t| = \gO(1)$ from \eqref{eq:11} it can be verified that EFL-FG obtains sub-linear regret of $\gO(\sum_{t=1}^{T}{(\ln K)^{\frac{3}{4}}\alpha(\gG_t)T^{-\frac{1}{4}}})$. It is useful to point out that $\alpha(\gG_t)$ depends on the budget $B_t$. Increase in $B_t$ can result in more connected feedback graph $\gG_t$ and as a result $\alpha(\gG_t)$ decreases. Therefore, larger budget can assist EFL-FG to achieve tighter sub-linear regret bound. For example, when the budget is large enough such that at each learning round $t$, the feedback graph $\gG_t$ is a densely connected graph as $\alpha(\gG_t)=\gO(1)$. In this case, the EFL-FG can achieve regret of $\gO((\ln K)^{\frac{3}{4}}T^{\frac{3}{4}})$. By contrast, when the feedback graph $\gG_t$ only includes self-loops, $\alpha(\gG_t)=K$ and as a result the EFL-FG can achieve regret of $\gO((\ln K)^{\frac{3}{4}}KT^{\frac{3}{4}})$.


\section{Experiments}
We tested the performance of different ensemble federated learning methods: our proposed EFL-FG and FedBoost \cite{Hamer2020} over the following real data sets downloaded from UCI machine learning repository \cite{Dua2017}:\\
\textbf{Bias Correction}: This dataset includes $7,750$ samples of air temperature information with $21$ features such as maximum or minimum air temperatures in the day. The goal is to predict the next-day minimum air temperature \cite{Cho2020}.  \\ 
\textbf{CCPP}: The dataset contains $9,568$ samples, with $4$ features including temperature, pressure, etc, collected from a combined cycle power plant. The goal is to predict hourly electrical energy output \cite{Tufekci2014}.\\
 \textbf{Energy}:  This dataset contains $19,735$ samples of $27$ features of house temperature and humidity conditions were monitored with a wireless sensor network. The goal is to predict the energy use of appliances. \cite{Candanedo2017}.

Consider the case where there are $100$ clients performing a regression task. 
The server stores $22$ pre-trained models, including kernel based regression models with $5$ Gaussian kernels, $5$ Laplacian kernels, $5$ polynomial kernels, $5$ sigmoid kernels and $2$ feed-forward neural networks. The bandwidth of Gaussian, Laplacian kernels, and the slope of sigmoid kernels are $0.01, 0.1, 1, 10, 100$. Also, the degree of polynomial kernels are $1, 2, 3, 4, 5$. The feedforward neural networks have $1$ and $2$ hidden layers respectively, where each hidden layer consists of $25$ neurons with ReLU activation functions. Each model is trained with $10\%$ of each dataset. Furthermore, the cost of sending each model is considered to be the number of parameters associated with model divided by the number of parameters associated with the model with maximum number of parameters. Hence, the maximum cost of sending a model is $1$. The budget  for sending models is $B=3$. 
The learning rate $\eta$ and exploration rate $\xi$ for all methods are set to be $\frac{1}{\sqrt{T}}$. The performance of ensemble federated learning methods is evaluated based on mean square error (MSE) at learning round $t$ defined as
$
    \text{MSE}_t = \frac{1}{t}\sum_{\tau=1}^{t}{\frac{1}{|\sC_\tau|}\sum_{i \in \sC_\tau}{(\hat{y}_{i,\tau}-y_{i,\tau})^2}} 
$
where $\hat{y}_{i,t}$ denote the prediction made by the $i$-th client at learning round $t$. In order to derive dominating sets for feedback graphs in \Algref{alg:2}, greedy set cover algorithm is employed \cite{Chvatal1979}. Note that in this experiments it is assumed that clients are not able to store their observed data in batch. Therefore, FedBoost implementation is modified to  cope with this situation for fair comparison. Specifically, at learning round $t$, the $i$-th client transmits its update with respect to newly observed sample $\vx_{i,t}$ instead of the whole batch of data.

\begin{table}[t]
    \setlength{\tabcolsep}{4pt}
	\caption{MSE ($\times 10^{-3}$) performance and percentage of budget violence.}
	\label{table:1}
	\vskip 0.05in
	\begin{center}
		\begin{small}
			\begin{tabular}{l||lll|lll}
				\toprule
				    &\multicolumn{3}{c}{MSE($\times 10^{-3}$)}  &\multicolumn{3}{c}{budget violence ($\%$)} \\
				Algorithms & Bias & CCPP & Energy   & Bias & CCPP & Energy \\
				\midrule
				FedBoost    & $69.11$ & $45.06$ & $10.19$ & $22.86\%$ & $13.79\%$ & $24.16\%$ \\
				\midrule
				EFL-FG  & $4.81$ & $4.92$ & $8.36$ & $0\%$ & $0\%$ & $0\%$\\
				\bottomrule
			\end{tabular}
		\end{small}
	\end{center}
	\vskip -0.1in
\end{table}

\begin{figure}
	\centering
	\vskip 0.2in
	\includegraphics[width=0.8\linewidth]{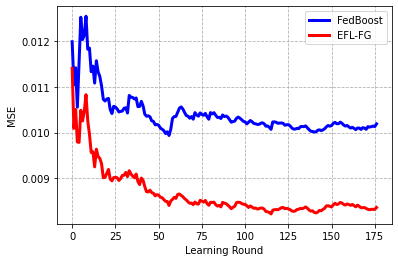}
	\vskip -0.1in
	\caption{MSE performance on Energy dataset.}
	\vskip -0.1in
	\label{fig:1}
\end{figure}

Tables \ref{table:1} shows the MSE and budget violence performance of different algorithms over all datasets, where \emph{budget violence} indicates the proportion of learning rounds when the instantaneous budget constraint is violated. As can be seen from Table \ref{table:1},  the proposed EFL-FG achieves lower MSE compared with FedBoost. Table \ref{table:1} shows that FedBoost violates the instantaneous budget in at least $13\%$ of learning rounds for all datasets, while EFL-FG guaratees that budget constraints is satisfied in every learning round.

Figure \ref{fig:1} illustrates the MSE  versus learning rounds in Energy dataset. It can be observed that our proposed EFL-FG outperforms FedBoost. 
Note that all algorithms obtain different MSE performance starting from the first learning round since different subsets of pre-trained models are chosen and combined by different algorithms.

\section{Conclusion}
The present paper developed a federated learning approach to learn ensemble of pre-trained models when the server cannot transmit all models to clients due to limitations in communication bandwidth and clients' memory. Specifically, the server generates a graph at each learning round and constructs an ensemble model by choosing a subset of pre-trained models using the graph. This paper provided the algorithm EFL-FG which constructs an ensemble model whose size does not surpass a certain budget at each learning round. We proved that EFL-FG achieves sub-linear regret. Experiments on several real datasets reveal the merits of EFL-FG compared with other ensemble federated learning state-of-art.

\bibliographystyle{IEEEtran}
\bibliography{References}

\appendix
\section{Proof of Theorem \ref{th:1}} \label{A}
In order to proof Theorem \ref{th:1}, the following Lemmas are proved and used as step-stone.
\begin{lemma} \label{lem:1}
Let $\hat{f}_k(.)$ represents the ensemble model associated with $v_k \in \gV$. The regret of $\hat{f}_k(.)$ with respect to $f_k(.)$ is bounded above as
\begin{align}
    & \sum_{t=1}^{T}{\sum_{i \in \sC_t}{\gL(\hat{f}_k(\vx_{i,t}),y_{i,t})}} - \sum_{t=1}^{T}{\sum_{i \in \sC_t}{\gL(f_k(\vx_{i,t}),y_{i,t})}} \nonumber \\ \le & \frac{\ln |\sN_{k,1}^\text{out}|}{\eta} + \frac{\eta}{2}\sum_{t=1}^{T}{\sum_{j \in \sN_{k,t}^\text{out}}{\frac{w_{j,t}|\sC_t|^2}{q_{j,t}W_{k,t}}}} \label{eq:1ap}
\end{align}
which holds for all $v_k \in \gV$.
\end{lemma}
\begin{proof}
Let $W_{k,t} = \sum_{j \in \sN_{k,t}^\text{out}}{w_{j,t}}$. According to definition of the set $\sM_{k,t}$ and the \Algref{alg:1}, it can be written that $W_{k,t+1} \le \sum_{j \in \sN_{k,t}^\text{out}}{w_{j,t+1}}$. Thus, for each $v_k \in \gV$ we find
\begin{align}
    \frac{W_{k,t+1}}{W_{k,t}} \le \sum_{j \in \sN_{k,t}^\text{out}}{\frac{w_{j,t+1}}{W_{k,t}}} = \sum_{j \in \sN_{k,t}^\text{out}}{\frac{w_{j,t}}{W_{k,t}}\exp(-\eta \ell_{j,t})}. \label{eq:2ap}
\end{align}
Using the inequality $e^{-x} \le 1-x+\frac{1}{2}x^2, \forall x \ge 0$, from \eqref{eq:2ap}, it can be concluded that
\begin{align}
    \frac{W_{k,t+1}}{W_{k,t}} \le \sum_{j \in \sN_{k,t}^\text{out}}{\frac{w_{j,t}}{W_{k,t}}(1-\eta \ell_{j,t}+\frac{1}{2}\eta^2 \ell_{j,t}^2)}. \label{eq:3ap}
\end{align}
Employing the inequality $1+x\le e^x$ and taking logarithm from both sides of \eqref{eq:3ap}, we obtain
\begin{align}
    \ln \frac{W_{k,t+1}}{W_{k,t}} \le \sum_{j \in \sN_{k,t}^\text{out}}{\frac{w_{j,t}}{W_{k,t}}(-\eta \ell_{j,t}+\frac{1}{2}\eta^2 \ell_{j,t}^2)}. \label{eq:4ap}
\end{align}
Summing \eqref{eq:4ap} over learning rounds leads to
\begin{align}
    \ln \frac{W_{k,T+1}}{W_{k,1}} \le \sum_{t=1}^{T}{\sum_{j \in \sN_{k,t}^\text{out}}{\frac{w_{j,t}}{W_{k,t}}(-\eta \ell_{j,t}+\frac{1}{2}\eta^2 \ell_{j,t}^2)}}. \label{eq:5ap}
\end{align}
In addition, $\ln \frac{W_{k,T+1}}{W_{k,1}}$ can be bounded from below as
\begin{align}
    \ln \frac{W_{k,T+1}}{W_{k,1}} \ge \ln \frac{w_{k,T+1}}{W_{k,1}} = -\eta \sum_{t=1}^{T}{\ell_{k,t}} - \ln |\sN_{k,1}^\text{out}|. \label{eq:6ap}
\end{align}
Combining \eqref{eq:5ap} with \eqref{eq:6ap}, we get
\begin{align}
    & \sum_{t=1}^{T}{\sum_{j \in \sN_{k,t}^\text{out}}{\frac{w_{j,t}}{W_{k,t}} \ell_{j,t}}} - \sum_{t=1}^{T}{\ell_{k,t}} \nonumber \\ \le & \frac{\ln |\sN_{k,1}^\text{out}|}{\eta} + \frac{\eta}{2}\sum_{t=1}^{T}{\sum_{j \in \sN_{k,t}^\text{out}}{\frac{w_{j,t}}{W_{k,t}}\ell_{j,t}^2}}. \label{eq:7ap}
\end{align}
According to \eqref{eq:5} and considering the fact that $\gL(f_k(\vx_{i,t}),y_{i,t}) \le 1$, taking expectation from $\ell_{k,t}$ and $\ell_{k,t}^2$, $\forall v_k \in \gV$, given observed losses in prior learning rounds it can concluded that
\begin{subequations} \label{eq:8ap}
    \begin{align}
        \E_t[\ell_{k,t}] &= \sum_{j \in \sN_{k,t}^\text{in}}{p_{j,t}\frac{\sum_{i \in \sC_t}{\gL(f_k(\vx_{i,t}),y_{i,t})}}{q_{k,t}}} \nonumber \\ &= \sum_{i \in \sC_t}{\gL(f_k(\vx_{i,t}),y_{i,t})} \label{eq:8apa} \\
        \E_t[\ell_{k,t}^2] &= \sum_{j \in \sN_{k,t}^\text{in}}{p_{j,t}\frac{\left(\sum_{i \in \sC_t}{\gL(f_k(\vx_{i,t}),y_{i,t})}\right)^2}{q_{k,t}^2}} \nonumber \\ &= \frac{\left(\sum_{i \in \sC_t}{\gL(f_k(\vx_{i,t}),y_{i,t})}\right)^2}{q_{k,t}} \le \frac{|\sC_t|^2}{q_{k,t}}. \label{eq:8apb}
    \end{align}
\end{subequations}
Taking the expectation from both sides of \eqref{eq:7ap} results in
\begin{align}
    & \sum_{t=1}^{T}{\sum_{i \in \sC_t}{\sum_{j \in \sN_{k,t}^\text{out}}{\frac{w_{j,t}}{W_{k,t}} \gL(f_j(\vx_{i,t}),y_{i,t})}}} - \sum_{t=1}^{T}{\sum_{i \in \sC_t}{\gL(f_k(\vx_{i,t}),y_{i,t})}} \nonumber \\ \le & \frac{\ln |\sN_{k,1}^\text{out}|}{\eta} + \frac{\eta}{2}\sum_{t=1}^{T}{\sum_{j \in \sN_{k,t}^\text{out}}{\frac{w_{j,t}|\sC_t|^2}{q_{j,t}W_{k,t}}}}. \label{eq:9ap}
\end{align}
Based on the Jensen's inequality and the convexity of the loss function $\gL(.,.)$, we can write
\begin{align}
    \sum_{j \in \sN_{k,t}^\text{out}}{\frac{w_{j,t}}{W_{k,t}} \gL(f_j(\vx_{i,t}),y_{i,t})} \ge& \gL(\sum_{j \in \sN_{k,t}^\text{out}}{\frac{w_{j,t}}{W_{k,t}}f_j(\vx_{i,t})},y_{i,t}) \nonumber \\ &= \gL(\hat{f}_k(\vx_{i,t}),y_{i,t}). \label{eq:10ap}
\end{align}
It can then be concluded from \eqref{eq:9ap} and \eqref{eq:10ap} that
\begin{align}
    & \sum_{t=1}^{T}{\sum_{i \in \sC_t}{\gL(\hat{f}_k(\vx_{i,t}),y_{i,t})}} - \sum_{t=1}^{T}{\sum_{i \in \sC_t}{\gL(f_k(\vx_{i,t}),y_{i,t})}} \nonumber \\ \le & \frac{\ln |\sN_{k,1}^\text{out}|}{\eta} + \frac{\eta}{2}\sum_{t=1}^{T}{\sum_{j \in \sN_{k,t}^\text{out}}{\frac{w_{j,t}|\sC_t|^2}{q_{j,t}W_{k,t}}}} \label{eq:11ap}
\end{align}
which proves the Lemma \ref{lem:1}.
\end{proof}
\begin{lemma} \label{lem:2}
The expected regret of EFL-FG with respect to $\hat{f}_{k}(.)$, $\forall k \in [K]$ is bounded as
\begin{align}
    & \sum_{t=1}^{T}{\sum_{i \in \sC_t}{\E_t[\gL(\hat{f}(\vx_{i,t}),y_{i,t})]}} - \sum_{t=1}^{T}{\sum_{i \in \sC_t}{\gL(\hat{f}_k(\vx_{i,t}),y_{i,t})}} \nonumber \\ \le & \frac{\ln K}{\eta} + \sum_{t=1}^{T}{\left(\xi(1-\frac{\eta}{2}|\sC_t|^2) + \frac{\eta}{2}K|\sC_t|^2\right)} \label{eq:12ap}
\end{align}
\end{lemma}
\begin{proof}
Recall $U_t = \sum_{k=1}^{K}{u_{k,t}}$. Based on the inequality $e^{-x} \le 1-x+\frac{1}{2}x^2, \forall x \ge 0$ we can write
\begin{align}
    \frac{U_{t+1}}{U_t} &= \sum_{k=1}^{K}{\frac{u_{k,t+1}}{U_t}} = \sum_{k=1}^{K}{\frac{u_{k,t}}{U_t}\exp(-\eta \hat{\ell}_{k,t})} \nonumber \\ & \le \sum_{k=1}^{K}{\frac{u_{k,t}}{U_t}(1-\eta \hat{\ell}_{k,t}+\frac{\eta^2}{2}\hat{\ell}_{k,t}^2 )}. \label{eq:13ap}
\end{align}
According to \eqref{eq:3}, the inequality in \eqref{eq:13ap} can be rewritten as
\begin{align}
    \frac{U_{t+1}}{U_t} \le \sum_{k=1}^{K}{\frac{p_{k,t}-\frac{\xi}{|\sD_t|}\gI(k \in \sD_t)}{1-\xi}(1-\eta \hat{\ell}_{k,t}+\frac{\eta^2}{2}\hat{\ell}_{k,t}^2 )} \label{eq:14ap}
\end{align}
Using the inequality $1+x \le e^x$, taking the logarithm from both sides of \eqref{eq:14ap} and summing the result over learning rounds, it can be inferred that
\begin{align}
    & \ln \frac{U_{T+1}}{U_1} \nonumber \\ \le & \sum_{t=1}^{T}{\sum_{k=1}^{K}{\frac{p_{k,t}-\frac{\xi}{|\sD_t|}\gI(k \in \sD_t)}{1-\xi}(-\eta \hat{\ell}_{k,t}+\frac{\eta^2}{2}\hat{\ell}_{k,t}^2 )}}. \label{eq:15ap}
\end{align}
Furthermore, $\ln \frac{U_{T+1}}{U_1}$ can be bounded from below as follows
\begin{align}
    \ln \frac{U_{T+1}}{U_1} \ge \ln \frac{u_{k,T+1}}{U_1} = -\eta \sum_{t=1}^{T}{\hat{\ell}_{k,t}} - \ln K. \label{eq:16ap}
\end{align}
Combining \eqref{eq:15ap} with \eqref{eq:16ap} leads to
\begin{align}
    & \sum_{t=1}^{T}{\sum_{k=1}^{K}p_{k,t}\hat{\ell}_{k,t}} - \sum_{t=1}^{T}{\hat{\ell}_{k,t}} \nonumber \\ \le & \frac{\ln K}{\eta} + \sum_{t=1}^{T}{\sum_{k=1}^{K}{\frac{\xi \gI(k \in \sD_t)}{|\sD_t|}\hat{\ell}_{k,t}}} \nonumber \\ & + \sum_{t=1}^{T}{\sum_{k=1}^{K}{\frac{\eta}{2}\left(p_{k,t}-\frac{\xi}{|\sD_t|}\gI(k \in \sD_t)\right)\hat{\ell}_{k,t}^2}}. \label{eq:17ap}
\end{align}
Expected values of $\hat{\ell}_{k,t}$ and $\hat{\ell}_{k,t}^2$ given prior observed losses can be expressed as
\begin{subequations} \label{eq:18ap}
    \begin{align}
    \E_t[\hat{\ell}_{k,t}] &= p_{k,t}\frac{\sum_{i \in \sC_t}{\gL(\hat{f}_k(\vx_{i,t}),y_{i,t})}}{p_{k,t}} \nonumber \\ &= \sum_{i \in \sC_t}{\gL(\hat{f}_k(\vx_{i,t}),y_{i,t})} \label{eq:18apa} \\
    \E_t[\hat{\ell}_{k,t}^2] &= p_{k,t}\frac{(\sum_{i \in \sC_t}{\gL(\hat{f}_k(\vx_{i,t}),y_{i,t})})^2}{p_{k,t}^2} \nonumber \\ &= \frac{(\sum_{i \in \sC_t}{\gL(\hat{f}_k(\vx_{i,t}),y_{i,t})})^2}{p_{k,t}} \le \frac{|\sC_t|^2}{p_{k,t}}. \label{eq:18apb}
    \end{align}
\end{subequations}
Taking the expectation from both sides of the \eqref{eq:17ap}, we get
\begin{align}
    & \sum_{t=1}^{T}{\sum_{i \in \sC_t}{\sum_{k=1}^{K}{p_{k,t}\gL(\hat{f}_k(\vx_{i,t}),y_{i,t})}}} - \sum_{t=1}^{T}{\sum_{i \in \sC_t}{\gL(\hat{f}_k(\vx_{i,t}),y_{i,t})}} \nonumber \\ \le & \frac{\ln K}{\eta} + \sum_{t=1}^{T}{\sum_{k=1}^{K}{\frac{\xi \gI(k \in \sD_t)}{|\sD_t|}\gL(\hat{f}_k(\vx_{i,t}),y_{i,t})}} \nonumber \\ &+ \sum_{t=1}^{T}{\sum_{k=1}^{K}{\frac{\eta}{2}(p_{k,t}-\frac{\xi}{|\sD_t|}\gI(k \in \sD_t))\frac{|\sC_t|^2}{p_{k,t}}}}. \label{eq:19ap}
\end{align}
Taking into account that $\gL(\hat{f}_k(\vx_{i,t}),y_{i,t}) \le 1$ and $p_{k,t} \le 1$, we can conclude that
\begin{align}
    & \sum_{t=1}^{T}{\sum_{i \in \sC_t}{\sum_{k=1}^{K}{p_{k,t}\gL(\hat{f}_k(\vx_{i,t}),y_{i,t})}}} - \sum_{t=1}^{T}{\sum_{i \in \sC_t}{\gL(\hat{f}_k(\vx_{i,t}),y_{i,t})}} \nonumber \\ \le & \frac{\ln K}{\eta} + \sum_{t=1}^{T}{\xi(1-\frac{\eta}{2}|\sC_t|^2) + \frac{\eta}{2}K|\sC_t|^2}. \label{eq:20ap}
\end{align}
Moreover, according to \Algref{alg:2}, it can be inferred that
\begin{align}
    \E_t[\gL(\hat{f}(\vx_{i,t}),y_{i,t})] = \sum_{k=1}^{K}{p_{k,t}\gL(\hat{f}_k(\vx_{i,t}),y_{i,t})}. \label{eq:21ap}
\end{align}
Thus, we can write
\begin{align}
    & \sum_{t=1}^{T}{\sum_{i \in \sC_t}{\E_t[\gL(\hat{f}(\vx_{i,t}),y_{i,t})]}} - \sum_{t=1}^{T}{\sum_{i \in \sC_t}{\gL(\hat{f}_k(\vx_{i,t}),y_{i,t})}} \nonumber \\ \le & \frac{\ln K}{\eta} + \sum_{t=1}^{T}{\left(\xi(1-\frac{\eta}{2}|\sC_t|^2) + \frac{\eta}{2}K|\sC_t|^2\right)}. \label{eq:22ap}
\end{align}
which proves Lemma \ref{lem:2}.
\end{proof}
Combining Lemma \ref{lem:1} with Lemma \ref{lem:2}, the following inequality holds
\begin{align}
    & \sum_{t=1}^{T}{\sum_{i \in \sC_t}{\E_t[\gL(\hat{f}(\vx_{i,t}),y_{i,t})]}} - \sum_{t=1}^{T}{\sum_{i \in \sC_t}{\gL(f_k(\vx_{i,t}),y_{i,t})}} \nonumber \\ \le & \frac{\ln (K|\sN_{k,1}^\text{out}|)}{\eta} + \frac{\eta}{2}\sum_{t=1}^{T}{\sum_{j \in \sN_{k,t}^\text{out}}{\frac{w_{j,t}|\sC_t|^2}{q_{j,t}W_{k,t}}}} \nonumber \\ &+ \sum_{t=1}^{T}{\left(\xi(1-\frac{\eta}{2}|\sC_t|^2) + \frac{\eta}{2}K|\sC_t|^2\right)} \label{eq:23ap}
\end{align}
which holds for all $v_k \in \gV$. Therefore, we can conclude that
\begin{align}
    & \sum_{t=1}^{T}{\sum_{i \in \sC_t}{\E_t[\gL(\hat{f}(\vx_{i,t}),y_{i,t})]}} - \min_{k \in [K]}\sum_{t=1}^{T}{\sum_{i \in \sC_t}{\gL(f_k(\vx_{i,t}),y_{i,t})}} \nonumber \\ \le & \frac{\ln (K|\sN_{k^*,1}^\text{out}|)}{\eta} + \frac{\eta}{2}\sum_{t=1}^{T}{\sum_{j \in \sN_{k^*,t}^\text{out}}{\frac{w_{j,t}|\sC_t|^2}{q_{j,t}W_{k^*,t}}}} \nonumber \\ &+ \sum_{t=1}^{T}{\left(\xi(1-\frac{\eta}{2}|\sC_t|^2) + \frac{\eta}{2}K|\sC_t|^2\right)} \label{eq:24ap}
\end{align}
which proves the Theorem \ref{th:1}.

\end{document}

%% file: math_commands.tex
\usepackage{amsmath,amsfonts,bm}

\def\Algref#1{Algorithm~\ref{#1}}

\def\1{\bm{1}}

\def\vp{{\bm{p}}}

\def\vw{{\bm{w}}}
\def\vx{{\bm{x}}}

\DeclareMathAlphabet{\mathsfit}{\encodingdefault}{\sfdefault}{m}{sl}
\SetMathAlphabet{\mathsfit}{bold}{\encodingdefault}{\sfdefault}{bx}{n}

\def\gE{{\mathcal{E}}}

\def\gG{{\mathcal{G}}}

\def\gI{{\mathcal{I}}}

\def\gL{{\mathcal{L}}}

\def\gO{{\mathcal{O}}}

\def\gR{{\mathcal{R}}}
\def\gS{{\mathcal{S}}}

\def\gV{{\mathcal{V}}}

\def\gX{{\mathcal{X}}}
\def\gY{{\mathcal{Y}}}

\def\sC{{\mathbb{C}}}
\def\sD{{\mathbb{D}}}

\def\sM{{\mathbb{M}}}
\def\sN{{\mathbb{N}}}

\def\sS{{\mathbb{S}}}

\newcommand{\E}{\mathbb{E}}



%% file: main.bbl
\begin{thebibliography}{10}
\providecommand{\url}[1]{#1}
\csname url@samestyle\endcsname
\providecommand{\newblock}{\relax}
\providecommand{\bibinfo}[2]{#2}
\providecommand{\BIBentrySTDinterwordspacing}{\spaceskip=0pt\relax}
\providecommand{\BIBentryALTinterwordstretchfactor}{4}
\providecommand{\BIBentryALTinterwordspacing}{\spaceskip=\fontdimen2\font plus
\BIBentryALTinterwordstretchfactor\fontdimen3\font minus
  \fontdimen4\font\relax}
\providecommand{\BIBforeignlanguage}[2]{{%
\expandafter\ifx\csname l@#1\endcsname\relax
\typeout{** WARNING: IEEEtran.bst: No hyphenation pattern has been}%
\typeout{** loaded for the language `#1'. Using the pattern for}%
\typeout{** the default language instead.}%
\else
\language=\csname l@#1\endcsname
\fi
#2}}
\providecommand{\BIBdecl}{\relax}
\BIBdecl

\bibitem{McMahan2017}
B.~McMahan, E.~Moore, D.~Ramage, S.~Hampson, and B.~A. y~Arcas,
  ``{Communication-Efficient Learning of Deep Networks from Decentralized
  Data},'' in \emph{Proceedings of International Conference on Artificial
  Intelligence and Statistics}, vol.~54, Apr 2017, pp. 1273--1282.

\bibitem{konency2017}
J.~Konečný, H.~B. McMahan, F.~X. Yu, P.~Richtárik, A.~T. Suresh, and
  D.~Bacon, ``Federated learning: Strategies for improving communication
  efficiency,'' 2017.

\bibitem{Rothchild2020}
D.~Rothchild, A.~Panda, E.~Ullah, N.~Ivkin, I.~Stoica, V.~Braverman,
  J.~Gonzalez, and R.~Arora, ``{F}etch{SGD}: Communication-efficient federated
  learning with sketching,'' in \emph{Proceedings of International Conference
  on Machine Learning}, vol. 119, Jul 2020, pp. 8253--8265.

\bibitem{Abad2019}
M.~S.~H. {Abad}, E.~{Ozfatura}, D.~{GUndUz}, and O.~{Ercetin}, ``Hierarchical
  federated learning across heterogeneous cellular networks,'' in \emph{IEEE
  International Conference on Acoustics, Speech and Signal Processing
  (ICASSP)}, 2020, pp. 8866--8870.

\bibitem{Buhlmann2012}
P.~B{\"u}hlmann, \emph{Bagging, Boosting and Ensemble Methods}.\hskip 1em plus
  0.5em minus 0.4em\relax Springer Berlin Heidelberg, 2012, pp. 985--1022.

\bibitem{Hamer2020}
J.~Hamer, M.~Mohri, and A.~T. Suresh, ``{F}ed{B}oost: A communication-efficient
  algorithm for federated learning,'' in \emph{Proceedings of International
  Conference on Machine Learning}, vol. 119, Jul 2020, pp. 3973--3983.

\bibitem{Chen2021}
X.~Chen, S.~Zhou, B.~Guan, K.~Yang, H.~Fao, H.~Wang, and Y.~Wang, ``Fed-eini:
  An efficient and interpretable inference framework for decision tree
  ensembles in vertical federated learning,'' in \emph{IEEE International
  Conference on Big Data (Big Data)}, 2021, pp. 1242--1248.

\bibitem{Cesa-Bianchi2006}
N.~Cesa-Bianchi and G.~Lugosi, \emph{Prediction, Learning, and Games}.\hskip
  1em plus 0.5em minus 0.4em\relax USA: Cambridge University Press, 2006.

\bibitem{Auer2003}
P.~Auer, N.~Cesa-Bianchi, Y.~Freund, and R.~E. Schapire, ``The nonstochastic
  multiarmed bandit problem,'' \emph{SIAM Journal on Computing}, vol.~32,
  no.~1, p. 48–77, Jan 2003.

\bibitem{Mannor2011}
S.~Mannor and O.~Shamir, ``From bandits to experts: On the value of
  side-observations,'' in \emph{Proceedings of International Conference on
  Neural Information Processing Systems}, 2011, pp. 684--692.

\bibitem{Cortes2020}
C.~Cortes, G.~DeSalvo, C.~Gentile, M.~Mohri, and N.~Zhang, ``Online learning
  with dependent stochastic feedback graphs,'' in \emph{Proceedings of
  International Conference on Machine Learning}, Jul 2020.

\bibitem{Chvatal1979}
V.~Chvatal, ``A greedy heuristic for the set-covering problem,''
  \emph{Mathematics of Operations Research}, vol.~4, no.~3, pp. 233--235, Aug
  1979.

\bibitem{Alon2017}
N.~Alon, N.~Cesa-Bianchi, C.~Gentile, S.~Mannor, Y.~Mansour, and O.~Shamir,
  ``Nonstochastic multi-armed bandits with graph-structured feedback,''
  \emph{SIAM Journal on Computing}, vol.~46, no.~6, pp. 1785--1826, 2017.

\bibitem{Dua2017}
D.~Dua and C.~Graff, ``{UCI} machine learning repository,'' 2017.

\bibitem{Cho2020}
D.~Cho, C.~Yoo, J.~Im, and D.-H. Cha, ``Comparative assessment of various
  machine learning-based bias correction methods for numerical weather
  prediction model forecasts of extreme air temperatures in urban areas,''
  \emph{Earth and Space Science}, vol.~7, no.~4, Mar 2020.

\bibitem{Tufekci2014}
P.~Tüfekci, ``Prediction of full load electrical power output of a base load
  operated combined cycle power plant using machine learning methods,''
  \emph{International Journal of Electrical Power and Energy Systems}, vol.~60,
  pp. 126 -- 140, 2014.

\bibitem{Candanedo2017}
L.~M. Candanedo, V.~Feldheim, and D.~Deramaix, ``Data driven prediction models
  of energy use of appliances in a low-energy house,'' \emph{Energy and
  Buildings}, vol. 140, pp. 81 -- 97, 2017.

\end{thebibliography}
